\documentclass{article}

\usepackage{microtype}
\usepackage{graphicx}
\usepackage{subfigure}
\usepackage{booktabs} 

\usepackage{hyperref}

\usepackage[accepted]{mlsys2022}

\mlsystitlerunning{Efficient Sparse Secure Aggregation for Federated Learning}

\usepackage{amssymb}
\usepackage{amstext}
\usepackage{amsmath}
\usepackage{amsthm}
\newtheorem{theorem}{Theorem}
\newtheorem{lemma}{Lemma}

\usepackage{pifont}
\newcommand{\cmark}{\ding{51}}%
\newcommand{\xmark}{\ding{55}}%

\begin{document}
\twocolumn[
\mlsystitle{Efficient Sparse Secure Aggregation for Federated Learning}



\mlsyssetsymbol{equal}{*}

\begin{mlsysauthorlist}
\mlsysauthor{Constance Beguier}{owkin}
\mlsysauthor{Mathieu Andreux}{owkin}
\mlsysauthor{Eric W. Tramel}{owkin}
\end{mlsysauthorlist}

\mlsysaffiliation{owkin}{Owkin Inc., New York, USA}

\mlsyscorrespondingauthor{Constance Beguier}{constance.beguier@owkin.com}

\mlsyskeywords{Machine Learning, Federated Learning, Secure Multiparty Computation}

\vskip 0.3in

\begin{abstract}
Federated Learning enables one to jointly train a machine learning model across distributed clients
holding sensitive datasets.
In real-world settings, this approach is hindered
by expensive communication and privacy concerns.
Both of these challenges have already been addressed individually, resulting in competing optimisations.
In this article, we tackle them simultaneously for one of the first times.
More precisely, we adapt compression-based federated techniques to additive secret sharing,
leading to an efficient secure aggregation protocol, with an adaptable security level.
We prove its privacy against malicious adversaries and its correctness in the semi-honest setting.
Experiments on deep convolutional networks demonstrate that
our secure protocol achieves high accuracy with low communication costs.
Compared to prior works on secure aggregation, our protocol has a lower communication and computation costs
for a similar accuracy.
\end{abstract}

]



\printAffiliationsAndNotice{}  

\section{Introduction}
Machine learning (ML) requires the collection of
large volumes of data in order to train robust predictive models.
In some healthcare applications, e.g. for rare diseases, this data collection
necessarily involves data stemming from different locations.
However, due to data sensitivity,
it may be forbidden or extremely difficult to centrally collect them.
Federated Learning (FL) introduced in~\cite{shokri2015privacy, FL_MMRA16}
is an approach to train an ML model that benefits from multiple datasets
while keeping training data in place.
FL algorithms typically iterate training rounds, during which model updates
are usually aggregated after local training steps.
While the approach is promising,
two major challenges hinder the large-scale
adoption of FL techniques in real-world use-cases:
privacy concerns and expensive communication~\cite{FL_chaellenges_LSTS20}.

Although leaving data at its source is a significant improvement for data privacy,
sharing intermediate model updates
indirectly leaks sensitive information~\cite{privacy_BDFK18,privacy_CLKE18,privacy_MSCS19}.
A standard protection technique is secure aggregation~\cite{PracticalSA_GIKM17,EaSTLy_DCSW20}.
With this approach, only the aggregated models are revealed,
while the local intermediate models are kept private,
which reduces the effectiveness of attacks.
Secure aggregation is often based on either Secure Multiparty Computation (SMC) protocols~\cite{SMC_EKR18}
or Homomorphic Encryption (HE)~\cite{HE_AAUC17}.
The main bottleneck of these techniques
is the additional computation and communication costs,
which amplify the burden of communication costs in FL.

Due to its distributed nature, communication costs are significant in FL.
In order to mitigate communication overhead,
multiple neural network update compression techniques
have been introduced in the last few years
\cite{Sto2015, SparseBinaryCompression_SWMS18, DoubleSqueeze_TYLZ19}.
These techniques demonstrate that only the most significant bits
of model updates' components are required to obtain a good predictive model
and very few of these components carry significant information.
It permits one to greatly reduce the communication costs
without compromising the final model accuracy.
Unfortunately, these techniques are not directly compatible
with efficient secure aggregation protocols,
which limits their practical impact.

In this paper, instead of addressing each challenge separately,
which leads to competing optimisations, we investigate them simultaneously,
for one of the first times.
Our main contribution is to introduce a set of new efficient protocols
for secure aggregation in collaborative FL (see Sec.~\ref{sec:proposed-method}),
which permits to achieve a good tradeoff between privacy and communication costs
according to the security level required in each practical use case.
These protocols combine neural network update compression techniques
\cite{Sto2015, SparseBinaryCompression_SWMS18, DoubleSqueeze_TYLZ19}
with additive secret sharing \cite{SecretSharingBook_CDN2015}.
They are provably private against malicious adversaries when at most all servers except one collude
(see Appendix~\ref{appendix:privacy-proof}),
have a low computation cost by reducing the secure aggregation
to the secure evaluation of only additions
(see Sec.~\ref{subsec:secure_aggregation}), and
reduce its communication cost thanks to update compression techniques
as demonstrated in Sec.~\ref{subsec:total_comm_cost}.
Our experiments on deep convolutional neural networks in MNIST and CIFAR-10
prove that our protocols obtain similar accuracy than
non-secure uncompressed FL training
with a lower communication cost (see Sec.~\ref{sec:experiments}).
Compared to secure aggregation techniques~\cite{PracticalSA_GIKM17,EaSTLy_DCSW20},
our protocols have a lower computation and communication costs for a similar accuracy.

\section{Background}
\label{sec:background}

\subsection{Notation}
In this paper, $[a,b]$ represents the set of integers $\{a, a+1, ..., b\}$.
$\mathbb{Z}_m$ (resp. $\mathbb{Z}_m^*$) is isomorphic to $[0, m-1 ]$ (resp. $[1, m-1]$).
The space of vectors of length $n$ whose components belong to $\mathbb{Z}_m$
is represented by $\mathbb{Z}_m^n$.
The notation $x \in_R S$ indicates that $x$ is sampled uniformly at random from the finite set $S$.
When $x \in_R \mathbb{Z}_m ^n$, all coordinates of $x$ are sampled independently and uniformly at random from $\mathbb{Z}_m$.
The notation $S \backslash i$ represents the set $S$ deprived of the element $i$:
$S \backslash i = \{k \in S~|~ k \neq i\}$.

We also introduce some notations related to the federated system.
Let $C$ denote the number of clients jointly training a neural network, which consists of $N$ trainable parameters.
Thus, the length of the vector of udpates is equal to $N$.
The clients jointly train this network during $R$ federated rounds.
In each federated round, each client performs $E$ local updates before aggregation.

\subsection{Federated Learning}

\subsubsection{Collaborative or Cross-Silo Setting}
Many current FL works such as~\cite{BEG2019,LLHJ2019}
study massive-scale use cases where millions of different devices participate to the training.
In contrast, in this manuscript, we focus on collaborative FL applied to medical use cases.
In this context, a small number of clients (generally less than 10) takes part in the federated training.
Clients are already able to obtain useful trained models from only their own local dataset
(which is generally moderately large),
but they would like to augment their performance through collaboration.
Unlike the massive-scale use-case, all clients are able to participate in each round with 
robust connectivity to the system.

\subsubsection{Federated Averaging~\cite{FL_MMRA16}}
In the studied FL setting, the goal is to obtain a \emph{single} trained model from some datasets
held by $C$ clients, $\mathcal{D}_i~\forall i \in [1,C]$, without moving
these datasets to a single location.
The most commonly used strategy in FL is Federated Averaging (\texttt{FedAvg}) \cite{FL_MMRA16}.
In this algorithm, within each such \emph{federated round}, indexed by $r \in [1,R]$,
each client $i$ is provided with the current global model state $\theta^{(r)}$,
performs $E$ local optimization steps $\text{Opt}_E$ to minimize its local loss $\mathcal{L}_i$,
and evaluates its local update $U_i^{(r)}$:
\begin{subequations}\label{eq:fl_training}
\begin{align}
    &\theta_i^{(0)} \gets \theta^{(r)} && \text{(Local Initial State),} \\ 
    &\theta_i^{(E)} \gets \text{Opt}_E(\theta_i^{(0)}, \mathcal{D}_i) && \text{(Local Training),}\\
    &U_i^{(r)} \gets \theta_i^{(E)} - \theta_i^{(0)} && \text{(Local Update)}
\end{align}
\end{subequations}
Then the local updates are transmitted to a central \emph{aggregator}
which maintains and updates the global model state, 
\begin{equation}\label{eq:aggregation}
\begin{aligned}
    \theta^{(r+1)} \gets  \theta^{(r)} + \sum_{i=1}^{C}\frac{|\mathcal{D}_i|}{D}~U_i^{(r)}, && \text{(Aggregate \& Update)}
\end{aligned}
\end{equation}
where $D=\sum_{i=1}^C |\mathcal{D}_i|$ is the total number of data samples across all datasets.
Tens to thousands federated rounds are computed in sequence until the global model state converges or 
meets some desired performance metric.

When multiple local update steps are used ($E>1$),
the convergence of \texttt{FedAvg} is not guaranted even in the convex setting
due to data heterogeneity~\cite{heterogeneity_SLSZ18,SCAFFOLD_KKMRSS19}.
The trade-off for \texttt{FedAvg}, then, is to select a large enough $E$ so as to reduce
the number of federated rounds required (and therefore communication overhead) to
reach a desired performance metric, but not so large that the algorithm fails to converge.

\subsubsection{Communication-Efficient FL}
Over the past few years, many model update compression techniques such as
\cite{SFD2014, Sto2015, SparseBinaryCompression_SWMS18, DoubleSqueeze_TYLZ19}
have shown that it is possible to greatly reduce the communication costs of distributed training
while retaining the predictive accuracy of the final model.
In~\cite{SFD2014}, gradients are encoded using only the direction (sign) of each coefficient.
This technique is further enhanced in~\cite{Sto2015}
where the authors demonstrate that very few of these components carry significant information, and thus many can be discarded,
resulting in a so-called \emph{ternary} coding of the model update~\cite{TernGrad_WXYW17}.
Both of these approaches make use of an \emph{error compensation} (EC) scheme to stabilize training.
EC has been studied analytically in~\cite{ErrorFeedback_KRSJ19,SparsifiedSGD_SCJ18},
where it was proven that this technique achieves the same theoretical convergence guarantees as standard training.
Recent works such as~\cite{SparseBinaryCompression_SWMS18, DoubleSqueeze_TYLZ19}
have showcased the utility of such model update compression schemes in practical federated settings.

\subsection{Additive Secret Sharing}
\label{subsec:SMC}

Secure Multiparty Computation (SMC) is a cryptography subfield which studies techniques allowing several parties to jointly
evaluate a function over their inputs while keeping these inputs private. 
Many protocols exist in this field, such as Yao garbled circuit \cite{Yao86} and secret sharing \cite{SS_RST79}.
These protocols have different trade-offs between the computation and communication costs.

For our secure aggregation, we choose to only rely on Additive Secret Sharing~\cite{SecretSharingBook_CDN2015},
owing to its low computation costs.
Its communication cost is proportional to the number of multiplications to evaluate.
To obtain an efficient secure aggregation, we adapt our aggregation such that
only additions need to be securely evaluated.

In this section, we review the Additive Secret Sharing protocol in Sec.~\ref{subsubsec:ASS},
as well as its communication and computation costs in Sec.~\ref{subsubsec:ASS_comm_comp_costs},
and justify its security aspects in Sec.~\ref{subsubsec:ASS_security}.

\subsubsection{Additive Secret Sharing~\cite{SecretSharingBook_CDN2015}}
\label{subsubsec:ASS}
Alg.~\ref{algo:secure-sum} presents in details
how to securely evaluate the sum of integers
held by~$C$ clients, with additive secret sharing,
thanks to the help of $S$ servers ($S \geq 2$).
In this protocol, each client possesses a vector $X_i \in \mathbb{Z}_m^n$,
and desires to know the sum of all client vectors,
$X = \sum_{i=1}^C X_i \mod m$,
but no client wants to reveal its own vector to other parties.
During the evaluation, inputs, intermediate results and outputs are held by 
$S$ servers in the form of secret shares.

\begin{algorithm}[h!]
\caption{SecureSum Protocol}\label{algo:secure-sum}
\begin{algorithmic}
\STATE {\bfseries Input:} Each client has a vector $X_i \in \mathbb{Z}_m^n$
\STATE {\bfseries Locally on each Client $i \in [1,C]$}
    \STATE \hspace{1em} \texttt{// Create \& Transmit Shares}
    \STATE \hspace{1em} $R_{i,j} \in _R \mathbb{Z}_m^n  \quad \forall j \in [1, S-1] $
    \STATE \hspace{1em} $R_{i, S} \gets X_i - \sum_{j=1}^{S-1} R_{i,j} \mod m$
    \STATE \hspace{1em} Transmit $R_{i,j}$ to $\text{Server } j$,  $\forall j \in [1, S]$
\STATE {\bfseries Locally on each Server $j \in [1,S]$}
    \STATE \hspace{1em} \texttt{// Aggregate received shares}
    \STATE \hspace{1em} $R_j \gets \sum_{i=1}^C R_{i,j} \mod m$
    \STATE \hspace{1em} Broadcast $R_j$ to all clients
\STATE {\bfseries Locally on each Client $i \in [1,C]$}
    \STATE \hspace{1em} \texttt{// Aggregate received shares}
    \STATE \hspace{1em} $Z \gets \sum_{j=1}^S R_j \mod m$
\STATE {\bfseries Output:} $\begin{cases}
                \text{Clients: } Z = \sum_{i=1}^C X_i \mod m\\
                \text{Servers: } \varnothing \end{cases}$
\end{algorithmic}
\end{algorithm}

The SecureSum protocol does not require that the parties playing the
role of the server are mutually exclusive from the client parties.
No matter the architecture, the proposed protocol will only reveal
to the clients the result of the secure function evaluation, with
the caveat that if a client is \emph{also} a server, then this party will
observe the result by virtue of being a client itsef.

\subsubsection{Communication and Computation Costs}
\label{subsubsec:ASS_comm_comp_costs}
Compared to other SMC protocols such as Yao garbled circuit~\cite{Yao86} and GMW~\cite{GMW87},
no expensive cryptographic operations are used in the SecureSum protocol.
All parties have only to evaluate additions modulo $m$.
Clients have an additional burden to produce $S-1$ random vectors in $\mathbb{Z}_m^n$,
but this additional computational cost is also minimal.
Thus, the computation cost of the SecureSum protocol is very low.

About the communication cost,
each client sends to each server a single vector in $\mathbb{Z}_m^n$ of
$n \cdot \lceil \log_2 m \rceil$ bits, and then,
each server sends back to each client one vector in $\mathbb{Z}_m^n$
of $n\cdot \lceil \log_2 m \rceil$ bits.
Thus, the \emph{total} communication cost required to execute this
protocol is
\begin{equation}
    \label{eq:secure_sum_comm_cost}
    2 \cdot S \cdot C \cdot n \cdot \lceil \log_2 m \rceil \text{ bits}.
\end{equation}

\subsubsection{Security Definitions}
\label{subsubsec:ASS_security}
The security of an SMC protocol depends on the assumptions that are made on adversaries.
In this paper, we only take into account \emph{semi-honest} or \emph{malicious}, \emph{static} adversaries.
A \emph{Semi-honest adversary} follows the protocol but tries to infer
as much information as possible from the observed messages.
On the contrary, a \emph{Malicious adversary} can use any kind of strategy to learn information,
including sending fallacious messages.
A \emph{Static adversary} selects the parties to corrupt at the beginning of the protocol execution.
This set of corrupted parties is fixed for the whole protocol execution.

An SMC protocol is determined to be secure if it is both correct and private.
A protocol is \emph{correct} if the output of the protocol is correct for the desired secure operation.
A protocol is \emph{private} if each party will learn their output and nothing else, except what they can infer from
their own input and their own output.

\begin{theorem}
The SecureSum protocol (see Alg.~\ref{algo:secure-sum}) is correct against semi-honest adversaries
and private against  a malicious adversary controlling at most $S-1$ servers and any number of clients.
\end{theorem}
\begin{proof}
It is well known that the SecureSum protocol is correct
and private against semi-honest adversaries~\cite{SecretSharingBook_CDN2015}.
We demonstrate in Appendix~\ref{appendix:privacy-proof} that this protocol is private
against a malicious adversary controlling at most $S-1$ servers and any number of clients.
\end{proof}

It means that a malicious adversary cannot learn any information about the input data of honest parties,
but it can modify messages such that the output is incorrect.
In collaborative FL use cases, 
data owners benefit from the trained global model
and therefore do not want to reduce their own performance.
Therefore, semi-honest correctness is sufficient for our use case.
In addition, all participants have an incentive to try to obtain information
about data held by other participants.
Hence, our protocol must be private against malicious adversaries.

\subsection{Fixed-Point Representation}
While ML models are very often parameterized by floating-point values,
SMC operations are applied in a finite set (e.g. $\mathbb{Z}_m$).
Thus, a fixed-point representation is required to convert these real values into the finite set.
In our secure aggregation protocol (see Sec.~\ref{subsec:secure_aggregation}),
we would like to securely add only \emph{non-negative} real values $x \in \mathbb{R}^+$.
Thus, we adapted the fixed-point representation presented in~\cite{FixedPoint_CDNR20}
to only \emph{non-negative} real values, $Q: \mathbb{R}^+ \rightarrow \mathbb{Z}_{2^{\lambda}}$:
\begin{equation}
    Q(x) = \lfloor 2^a \cdot x \rfloor \mod 2^{\lambda}.
\label{eq:fixed-point-represention}
\end{equation}
In practice, all values $x$ (inputs, intermediate results and outputs) belong to a finite interval.
The integer $\lambda$ is selected according to the wanted precision.
To avoid overflow issues, the integer $a$ is set according to the largest value resulting
from any step in the secure computation. 

\section{Proposed Method}
\label{sec:proposed-method}

In this paper, we focus on collaborative FL, a context often
encountered in medical applications of federated techniques.
The aim of this work is to present a secure FL technique which answers 
to the following four design constraints:
\begin{itemize}
    \item[\emph{i})] 
        \textbf{Perform the model update aggregation in a secure way}. 
        Nothing will be revealed to the aggregation servers.
        Only the clients will learn the aggregated updates, and thus whatever
        information can be inferred from these updates in conjunction with their
        local data and model updates.
        Specifically, our protocol must be:
        \begin{itemize}
            \item[\emph{a})] 
                \textbf{Private in malicious settings.} 
                Clients and servers may deviate from the defined protocol
                to obtain information about other clients data.
            \item[\emph{b})] 
                \textbf{Correct in semi-honest setting.} 
                Participants have an incentive to behave correctly
                in order to obtain a trained model with better 
                performance than their local model. 
        \end{itemize}
    \item[\emph{ii})] 
        \textbf{No compromise on accuracy.} 
        Our method must obtain a similar accuracy than 
        non-secure uncompressed \texttt{FedAvg}.
    \item[\emph{iii})] 
        \textbf{Computation Efficient.} 
        Our method must not use expensive cryptographic operations.
    \item[\emph{iv})] 
        \textbf{Communication Efficient.} 
        Our method must have a smaller communication cost 
        than non-secure uncompressed \texttt{FedAvg}
        and state of the art secure aggregation~\cite{PracticalSA_GIKM17,EaSTLy_DCSW20}.
\end{itemize}

In this section, we will first describe a communication-efficient FL algorithm
which is amenable to efficient secure aggregation in 
Sec.~\ref{subsec:comm-efficient-fl}.
Then, in Sec.~\ref{subsec:secure_aggregation},
we will present our efficient secure aggregation protocol,
whose the main block is the SecureSum protocol (see Alg.~\ref{algo:secure-sum}).
Our proposed method is summarized in Alg.~\ref{alg:smc-friendly-fl}. 

\begin{algorithm}[h!]
\caption{SMC-friendly Compressed Update FL}\label{alg:smc-friendly-fl}
\begin{algorithmic}
\STATE {\bfseries Input:}
    Number of rounds $R$, 
    number of local updates $E$, 
    initial global model state $\theta^{(0)}$, 
    local datasets $\mathcal{D}_i$,
    desired compression factor $\rho$.
\STATE {\bfseries Initialization (Error accumulator and model states)}
\STATE \hspace{1em} $\delta_i^{(0)} \gets \mathbf{0}, \quad \theta_i^{(0)} \gets \theta^{(0)},  \quad \forall i \in [1,C]$
\FOR{$r=0, 1, ..., R-1$}
    \STATE {\bfseries Locally on each Client $i \in [1,C]$}
    \STATE \hspace{1em} $U_i \gets {\rm Opt}_E\left(\theta_i^{(r)}, \mathcal{D}_i \right)  - \theta_i^{(r)}$ \hfill Eq. (\ref{eq:fl_training})
    \STATE \hspace{1em} $(\alpha_i, D_i) \gets \mathtt{TopBinary}(U_i + \delta_i^{(r)})$ \hfill Eq. (\ref{eq:TopBinary})
    \STATE \hspace{1em} $\delta_i^{(r+1)} \gets \delta_i^{(r)} + U_i - \alpha_i \cdot D_i$
    \STATE {\bfseries Via SMC between clients and servers}
    \STATE \hspace{1em} $D \leftarrow \mathtt{SecureSum}(D_1, \dots, D_C)$ \hfill Sec.~\ref{subsubsec:secure-unions}, \ref{subsubsec:secure-sign-sum}
    \STATE \hspace{1em} $\alpha \leftarrow \mathtt{SecureSum}(\alpha_1, \dots, \alpha_C)$ \hfill Sec. \ref{subsubsec:secure-factor-sum}
    \STATE {\bfseries Locally on each Client $i \in [1,C]$}
    \STATE \hspace{1em}$ \theta_i^{(r+1)} \gets \theta_i^{(r)} + \frac{1}{C^2} \alpha \cdot  D$ \hfill Eq. (\ref{eq:aggregation}) with Eq. (\ref{eq:separate_aggregation})
\ENDFOR
\end{algorithmic}
\end{algorithm}

\subsection{Communication-Efficient FL}
\label{subsec:comm-efficient-fl}

\subsubsection{\texttt{TopBinary} Coding}
Our method is based on the compression of the updates sent from clients to the aggregation server
with an error compensation to ensure convergence \cite{Sto2015, ErrorFeedback_KRSJ19, SparseBinaryCompression_SWMS18}.
For the compression method, we used a combination of the 
Top-$k$ sparsification~\cite{SparsifiedSGD_SCJ18} as well as 1-bit quantization~\cite{SignSGD_BWAA18}
which we note here as \texttt{TopBinary} coding.
The Top-$k$ sparsification jointly compresses a vector $X$ 
by retaining only the $k$ components with largest magnitude. 
More precisely, let $X \in \mathbb{R}^N$, then 
\begin{equation}
    \forall i \in [1,N] \text{, Top-}k(X)[\pi_i] =
    \begin{cases}   
        X[\pi_i], & \text{if}~ \pi_i \leq k,\\
        0, & \text{otherwise}.
    \end{cases}
\end{equation}
where $\pi$ is a sorting permutation of $[1, N]$ such that $\forall i \in [1, N-1]$, 
$| X[\pi_i] | \geq | X[\pi_{i+1}] |$. 
In order to compare same relative compression rate between models of varying architectures,
we will introduce the term $\rho \triangleq k / N$,
which is interpreted as a desired model update sparsity level.

To further reduce communication costs, we also employ one-bit quantization
to the $k$ significant model update coefficients \cite{SparseBinaryCompression_SWMS18,Sto2015}.
This scalar quantization maps each non-zero element of a vector $X$ 
to the binary set $\{-\alpha, \alpha\}$ according to their signs. 
More precisely, we define our \texttt{TopBinary} coder as
\begin{equation}\label{eq:TopBinary}
    \texttt{TopBinary}(X; k) = \alpha \cdot \underbrace{\text{sgn}\circ\text{Top-}k(X)}_{D},
\end{equation}
where 
$\alpha \triangleq \frac{|| X ||_2}{||\text{sgn}\circ\text{Top-}k(X)||_2} = \frac{1}{\sqrt{k}} ||X||_2$
is a scaling factor used to preserve the Euclidean norm of $X$. 
We note that the implementation of the proposed secure protocol does not depend tightly
on this construction of $\alpha$, and other definitions~\cite{OneBit_SFDL14,SparseBinaryCompression_SWMS18}
could be chosen, as well. 

The \texttt{TopBinary} coded vector can be split into two components: the scalar 
factor $\alpha$ and the vector of signs $D = \text{sgn} \circ \text{Top-}k(X)$. Furthermore, 
when $\rho$ is small, the ternary-valued $D$ can instead be represented by two 
$k$-length vectors, namely, the list of non-zero indices, $V = [\pi_1, \dots, \pi_k ]$,
and the signs of the coefficients at those non-zero locations,
$(D)_{|V} = [D[\pi_1], \dots, D[\pi_k] ]$. 

\subsubsection{Separate Aggregation}
Let $\alpha_i \cdot D_i$ the updates of the client $i$
obtained with the \texttt{TopBinary} encoding.
Usually a \emph{direct aggregation} (\texttt{DirectAgg}) is performed and
the aggregated updates is defined as
\begin{equation}
\begin{aligned}
\frac{1}{C} \sum_{i=1}^C \alpha_i D_i  && \text{(DirectAgg).}
\end{aligned}
\end{equation}
Unfortunately, the secure evaluation of this \texttt{DirectAgg}
requires the secure evaluation of multiplications,
which is costly compared to the secure evaluation of additions.
To have an efficient protocol,
we would like to securely evaluate \emph{only additions}.
Thus, we suggest an alternative to aggregate compressed model updates called 
\emph{separate aggregation} (\texttt{SepAgg}).
This approach separately aggregates the scaling factors $\alpha_i$ and the vector of signs $D_i$
as follows:
\begin{equation}
\begin{aligned}
    U = \frac{1}{C^2} \left(\sum_{i=1}^C \alpha_i\right) \left(\sum_{i=1}^C D_i\right) && \text{(SepAgg).}
\end{aligned}
\label{eq:separate_aggregation}
\end{equation}
We contrast these two approaches in Appendix~\ref{appendix:separate-aggregation},
where we demonstrate that when the norm of each original vector~$||X_i||_2$ is close
to the average norm~$\frac{1}{C}\sum_i ||X_i||_2$,
then the \texttt{SepAgg} is an adequate estimation of the \texttt{DirectAgg}.

With \texttt{SepAgg}, the convergence of the federated training is not proved. 
However in our experiments, federated training with this \texttt{SepAgg} provides
 similar trained model predictive performance than \texttt{FedAvg} training with \texttt{DirectAgg}.
The main advantage of this \texttt{SepAgg} is to be SMC-friendly
without compromising accuracy.

\subsection{Compressed Secure Aggregation Protocol}
\label{subsec:secure_aggregation}

\begin{table*}[!t]
    \small
    \centering
    \begin{tabular}{lcll}
        \toprule
        Protocol & Communication Cost (bits)& Clients Learn (besides $V$)  & Servers Learn \\
        \midrule
        \texttt{PlaintextUnion} & $2 \cdot C \cdot N$ & Nothing & Clients' list of indices $V_i$\\
        \texttt{PartialSecUnion} & $2 \cdot S \cdot C \cdot N \cdot \lceil \log_2 (C+1) \rceil $ & How many clients have selected each index & Nothing\\
        \texttt{SecUnion} & $2 \cdot S \cdot C \cdot N \cdot q$ & If they alone selected each index & Nothing\\
        \bottomrule
    \end{tabular}
    \caption{Comparison of secure union protocols for the proposed secure aggregation.}
    \label{tab:comp-secure-union}
\end{table*}

In this section, we explain how the SecureSum algorithm (Alg.~\ref{algo:secure-sum}) is used as a building block
to construct a secure protocol for our sparse aggregation protocol
described in Sec.~\ref{subsec:comm-efficient-fl}.
The problem is the following: at the aggregation step, each client $i$ has a vector $X_i = \alpha_i D_i \in \mathbb{R}^N$
and one would like to securely evaluate $U = \texttt{SepAgg}(X_1, \dots, X_C)$.
Our proposed secure aggregation protocol can be described in three steps:
\begin{align}
    &V \gets \bigcup_{i=1}^C V_i = \bigcup_{i=1}^C \{k \in [1,N]~|~D_i[k] \neq 0\} \tag{Union of Indices}\\
    &D \gets \sum_{i=1}^C (D_i)_{|V} \tag{Sum of Signs}\\
    &\alpha \gets Q^{-1} \left(\sum_{i=1}^C Q(\alpha_i) \right) \tag{Sum of Factors}
\end{align}
Finally, each client can localy evaluate $U = \frac{1}{C^2} \cdot \alpha \cdot D$.
During this protocol, each client learns $V, D, \alpha$ and $U$ and the servers learn nothing.

In the next sections, we will detail
why the intermediate steps of this protocol are secure (Sec.~\ref{subsubsec:revealing-info}),
then we will provide a number of different approaches for the secure union (Sec.~\ref{subsubsec:secure-unions}),
the protocol for a secure sum of sign vectors (Sec.~\ref{subsubsec:secure-sign-sum}),
and finally the protocol for the secure sum of scalar factors (Sec.~\ref{subsubsec:secure-factor-sum}).

\subsubsection{Security of Intermediate Outputs}
\label{subsubsec:revealing-info}
At the end of the secure aggregation round,
each client obtains the separate aggregated updates $U$ 
(see Eq.~\ref{eq:separate_aggregation}).
In this section, we will prove that from $U$,
each client can easily infer the intermediate outputs $V$, $D$ and $\alpha$.
Thus, revealing to the clients these intermediate outputs over the course of the protocol
does not provide any \emph{additional} information than what they could learn simply from
their knowledge of their own inputs and result of the secure protocol at its completion. 
Hence, the proposed stepwise approach is secure as long as its individual steps are secure.

From $U$, the client can infer the union of the list of indices via
$\tilde{V} = \{k \in [1,N]~|~U[k] \neq 0\}$,
where $\tilde{V} \subset V = \bigcup_{i=1}^C \{k \in [1, N]~|~D_i[k] \neq 0 \}$ and 
$V  \backslash \tilde{V} = \{k \in V~|~\sum_{i=1}^C D_i[k] = 0\}$.
Then, the client can deduce the sum of the scalar factors via
$\widetilde{\alpha} = C^2 \cdot \min_{k \in  V}~\left|U[k]\right|$,
where~$\widetilde{\alpha} = \alpha$ if at least one index of the sum of the signs 
is equal to~$-1$ or~$1$, which is the case when one index has been selected by 
exactly one client, an occurrence we expect to be highly likely.
Finally, the client can obtain the sum of the signs
via~$D = \left(\frac{C^2}{\alpha} \right)U$.

\subsubsection{Secure Union Protocols}
\label{subsubsec:secure-unions}
Now, we will define a set of protocols,
each with a different secure perimeter and communication cost trade-off (see Table~\ref{tab:comp-secure-union}),
for finding the union of indices.
In each protocol, each client $i$ has a list of indices $V_i$ whose components 
belong to $[1,N]$.
The clients would like to securely evaluate the union 
of the lists of indices $V = \bigcup_{i=1}^C V_i$.

\paragraph{\texttt{PartialSecUnion}}
In this protocol, 
each client will create a Boolean vector $B_i$ representing their list of indices $V_i$ such that $\forall k \in [1, N],$
\begin{equation}
    B_i[k]=\begin{cases}
        1, & \text{ if } k \in V_i, \\
        0, & \text{ otherwise }.
    \end{cases}
\end{equation}
Using this representation, the clients can securely calculate the union of indices through
the use of Alg.~\ref{algo:secure-sum} by evaluating the sum of their Boolean vectors $B_i$
in $\mathbb{Z}_{C+1}^N$.
At the end of the secure sum evaluation, each client will obtain $B=\sum_{i=1}^C B_i \mod (C+1)$,
from which they may easily infer the union of indices $V$ by 
selecting the positions of non-zero values in $B$
and the servers learn nothing.
In addition, $B$ reveals to the clients \emph{how many} clients have non-zero values at a
given index, however it does not reveal \emph{which} clients.

\paragraph{\texttt{SecUnion}}
Here, each client will create a vector $A_i$ representing
its list of indices $V_i$ such that $\forall k \in [1, N],$ 
\begin{equation}
    A_i[k]=\begin{cases}
        R_{i,k}, & \text{ if } k \in V_i \text{ where } R_{i,k} \in_R \mathbb{Z}_{2^q}^*, \\
        0, & \text{ otherwise }.
    \end{cases}
\end{equation}
Subsequently, the clients will securely evaluate the sum of the vectors $A_i$ in $\mathbb{Z}^N_{2^q}$
using Alg.~\ref{algo:secure-sum}.
At the end of the secure evaluation, each client will obtain $A=\sum_{i=1}^C A_i \mod 2^q$
from which each client can easily infer the union of indices by 
selecting the positions of non-zero values. 

When using the \texttt{SecUnion} protocol, 
the resulting union of indices may have some false negatives.
Specifically, when more than two clients have the same index, the sum
of random values $R_{i,k}$ will be equal to $0$ with probability~$1 / 2^q$,
causing $k$ to be omitted from $V$ when it should be present.
Increasing $q$ can reduce the occurrence of such false negatives,
but at the cost of increased communication cost.
We refer to Appendix~\ref{appendix:secure_union} for an analysis of the average amount of false negatives.
We note that, in experimentation (see Tab.~\ref{tab:all_xps}),
the training with our \texttt{SecUnion} protocol
is observed empirically to be quite robust to such dropped support. 

In terms of information leakage, if one client observes that the component 
$A[k]$ in the sum result is equal to 
their random component $R_{i,k}$, then the client can infer that with very high 
probability he is alone to have this index $k$.
We refer to Appendix~\ref{appendix:secure_union} for an analysis of this probability.

\begin{table*}[!t]
\centering
\small
\setlength\tabcolsep{-2pt}
\begin{tabular}{lcc}
\toprule
Protocol & Sec. Agg. & Communication Cost per Federated Round (bits) \\
\midrule
Uncompressed \texttt{FedAvg} & \xmark & $2 \cdot C \cdot 32 \cdot N$ \\
\texttt{FedAvg} w/ \texttt{TopBinary} \& \texttt{DirectAgg} &
\xmark & $C \cdot (32 \cdot (1 + |V|)  + 2 \cdot N + \lfloor \rho \cdot N \rfloor)$ \\
\texttt{FedAvg} w/ \texttt{TopBinary} \& \texttt{SepAgg} &
\xmark & $C \cdot (2 \cdot 32 + 2 \cdot N + \lfloor \rho \cdot N \rfloor + |V| \cdot \lceil \log_2(2C+1) \rceil)$ \\
\midrule
Our method w/ \texttt{NoUnion} & \cmark &
$\left(2 \cdot S \cdot C \cdot N \cdot \lceil \log_2 (2C+1) \rceil \right)
+ \left(2 \cdot S \cdot C \cdot 32 \right)$\\
Our method w/ \texttt{SecUnion} & \cmark &
$\left(2 \cdot S \cdot C \cdot N \cdot q \right)
+\left(2 \cdot S \cdot C \cdot |V| \cdot \lceil \log_2 (2C+1) \rceil \right)
+ \left(2 \cdot S \cdot C \cdot 32 \right)$\\
Our method w/ \texttt{PartialSecUnion} & \cmark &
$\left(2 \cdot S \cdot C \cdot N \cdot \lceil \log_2 (C+1) \rceil \right)
+\left(2 \cdot S \cdot C \cdot |V| \cdot \lceil \log_2 (2C+1) \rceil \right)
+ \left(2 \cdot S \cdot C \cdot 32 \right)$\\
Our method w/ \texttt{PlaintextUnion} & \cmark &
$\left(2 \cdot C \cdot N \right)
+\left(2 \cdot S \cdot C \cdot |V| \cdot \lceil \log_2 (2C+1) \rceil \right)
+ \left(2 \cdot S \cdot C \cdot 32 \right)$\\
\midrule
\cite{PracticalSA_GIKM17} & \cmark & $C \cdot \left( 2 C \cdot a_K + (5 C-4)a_S + 2N \lceil \log_2 T \rceil\right)$ \\
\cite{EaSTLy_DCSW20} with TSS & \cmark & $2 S C N \lceil \log_2 (2C + 1) \rceil$ \\
\cite{EaSTLy_DCSW20} with HE & \cmark & $4 C N (prec + pad)$ \\
\bottomrule
\end{tabular}
\caption{Communication costs per federated round for the different protocols
where $|V|$ is the size of the union of indices,
for \cite{PracticalSA_GIKM17}, $a_K = a_S = 256$ are the number of bits in a public key and in an encrypted share,
all operations are evaluated in $\mathbb{Z}_T$ with $T=2^{32}$,
for \cite{EaSTLy_DCSW20}, $prec = 24$ and $pad = 8$ are HE parameters used to avoid overflows.
\label{tab:comm-cost-estimations}}
\end{table*}

\paragraph{\texttt{PlaintextUnion}}
Here, the clients will only learn the union of indices, however
the first server will learn the list of indices of each client.
In this protocol, each client sends to the first server
the Boolean vector $B_i$ representing their list of indices.
The first server evaluates the union of indices by OR-ing the 
received Boolean vectors and transmitting back the resulting
Boolean vector to each client.

While such a protocol is not advisable for every setting, especially for
NLP tasks or recommender systems, where the knowledge of support locations
can be indicative of training features, 
the effectiveness of many white-box privacy attacks on distributed ML systems 
has only been demonstrated for exact, real-valued model updates, and not for
model update support. 
In the case of computer-vision models, such a security perimeter may
be acceptable to the system designer.

\paragraph{\texttt{NoUnion}}
If the compression rate is very low, the split of the sum of signs into a
union of indices and a sum on restricted signs vectors is not beneficial.
In this case, it is more efficient to perform the secure aggregation in two steps:
i) the secure sum of signs vectors~$D_i$ directly, as~$V=[1, N]$, and
ii) the secure sum of scalar factors.

\subsubsection{Secure Sum of Signs}
\label{subsubsec:secure-sign-sum}
In the second step of our compressed secure aggregation protocol,
each client restricts its dense vector of ternary values $D_i$
to the support resulting from the secure union of indices $V$, and
then together they securely evaluate the sum of these binary vectors of signs, $(D_i)_{|V}$.
More precisely, the parties would like to securely evaluate the sum of 
$C$ vectors of length equal to the size of
the union of indices~$|V|$, each component belongs to $\{-1, 0, 1\}$.
The resulting vector will be a vector of size $|V|$ 
whose components belong to~$[-C, C]$.

For the secure sum of signs,
first, each client will convert each component of its restricted 
vector from~$[-C, C]$ to~$[0, 2C]$ via 
\begin{equation}
    \text{T}(x) = \begin{cases}
    x, & \text{ if } x \in [0, C], \\
    x + (2C+1), & \text{ if } x \in [-C,-1].
    \end{cases}
\label{eq:signs-convert}
\end{equation}
Subsequently, the clients will securely evaluate the sum of these converted vectors in $\mathbb{Z}_{2C+1}$
using the secure sum protocol of Alg.~\ref{algo:secure-sum}.
Finally, each client will convert the components of the result vector from $[0, 2C]$ back to $[-C, C]$ via
\begin{equation}
    \text{T}^{-1}(x) = \begin{cases}
    x, & \text{ \vspace{-2em} if } x \in [0, C], \\
    x - (2C+1), & \text{ \vspace{-2em} if } x \in [C+1,2C].
    \end{cases}
\label{eq:signs-inverse-convert}
\end{equation}

\subsubsection{Secure Sum of Factors}
\label{subsubsec:secure-factor-sum}
In the third step of the compressed secure aggregation protocol, 
each client $i$ has a scalar factor $\alpha_i \in \mathbb{R}^+$
and would like to obtain the sum of all factors $\alpha = \sum_{i=1}^C \alpha_i$.
Firstly, each client will convert his scalar factor from $\mathbb{R}^+$
into $\mathbb{Z}_{2^{\lambda}}$ with the fixed point representation
(see Eq.~\eqref{eq:fixed-point-represention}).
When choosing the fixed-point parameters $a$ and $\lambda$,
we must account for potential overflows when summing together the $C$ scalar factors.
In practice, $\lambda$ is fixed to $32$ bits in order to have a reasonable 
precision without unnecessarily increasing the size of the working field. 
Subsequently, $a$ is chosen so as to avoid this potential overflow.

After this representation conversion, the clients will securely evaluate the 
sum of these converted factors in $\mathbb{Z}_{2^{\lambda}}$, using the
secure sum protocol of Alg.~\ref{algo:secure-sum}.
Finally, each client will convert the resulting factor back into the initial space $\mathbb{R}^+$.

\subsection{Communication Costs}
\label{subsec:total_comm_cost}
In this section, we estimate the total communication cost per federated round of
\texttt{FedAvg} without secure aggregation
and of our proposed method with secure aggregation,
as summarized in Table~\ref{tab:comm-cost-estimations}.
We choose a 32-bit floating-point representation for all ``unquantized'' real values.

\subsubsection{Non-secure FedAvg}
For non-secure uncrompressed \texttt{FedAvg}, each client sends its updates to one server ($32 \cdot N$ bits)
and then the server sends the aggregated updates to each client ($32 \cdot N$ bits).
Thus, the total communication cost for one aggregation is equal to 
$2 \cdot C \cdot 32 \cdot N \text{ bits}$.

For non-secure \texttt{FedAvg} with \texttt{TopBinary} compression, 
each client sends to one server its factor~$\alpha$ ($32$ bits), its list of non zero indices~$V$ ($N$ bits)
and the signs of the coefficients at those non-zero locations~$(D)_{|V}$ ($\lfloor \rho \cdot N \rfloor $ bits).
Let~$|V|$ the size of the union of indices~$V$.
With \texttt{DirectAgg}, the server sends to each client the aggregated updates composed of
the list of the non-zero indices ($N$ bits) and the coefficients at those non-zero locations ($32 \cdot |V|$ bits).
For \texttt{SepAgg}, the server sends to each client the aggregated updates composed of
the aggregated factor ($32$ bits) and the aggregated vector of directions ($N + |V| \cdot \lceil \log_2 (2C+1) \rceil$ bits).

\subsubsection{Our method with secure aggregation}

We can easily evaluate the total communication cost in bits of each sub-protocol of our secure aggregation
from Eq. \eqref{eq:secure_sum_comm_cost}:
\begingroup
\allowdisplaybreaks
\begin{align}
    2 \cdot S \cdot C \cdot N \cdot \lceil \log_2 (C+1) \rceil \tag{\texttt{PartialSecUnion}} \\
    2 \cdot S \cdot C \cdot N \cdot q \tag{\texttt{SecUnion}} \\
    2 \cdot C \cdot N \tag{\texttt{PlaintextUnion}} \\
    2 \cdot S \cdot C \cdot |V| \cdot \lceil \log_2 (2C+1) \rceil \tag{\texttt{Sum of Signs}} \\
    2 \cdot S \cdot C \cdot 32 \tag{\texttt{Sum of Factors}}
\end{align}
\endgroup
where $|V|$ is the size of the union of indices $V$.

\section{Related Work}

To date, proposed secure aggregation schemes in FL have relied on both SMC and HE approaches.
Secure aggregation methods based on HE~\cite{Hybrid_TBAS19,BatchCrypt_ZLXW20}
require a single round of communication, but also require computationally expensive cryptographic operations.
On the contrary, secure aggregation methods based on SMC~\cite{PracticalSA_GIKM17,HybridAlpha_XBZA19,SecAgg_BBGL20}
have a lower computational cost, but a higher communication cost.

A secure aggregation protocol based on Threshold Secret Sharing (TSS) is introduced in~\cite{PracticalSA_GIKM17}.
Contrary to our proposed method (see Section~\ref{sec:proposed-method}),
this protocol requires the use of a single aggregation server,
this server learns the aggregated models
and this protocol is resilient against clients dropout.
With respect to computational requirements, this protocol relies on key agreements, 
asymmetric cryptography, and the use of
cryptographically-secure pseudo-random number generators (PRNG). 
These cryptographic operations are computationally expensive and increase the computational burden
for each participating party.
To be resilient against dropouts, in the protocol of~\cite{PracticalSA_GIKM17},
public keys and encrypted shares are routed from one client to another through the central
coordination server, incurring extra communication costs.
More precisely, for one round, the total communication cost is
\begin{equation}
C \cdot \left( 2 C \cdot a_K + (5 C-4)a_S + 2N \lceil \log_2 T \rceil \right) \text{ bits,}
\end{equation}
where $a_K=a_S=256$ are the number of bits respectively in a public key and in an encrypted share
and all operations are evaluated in $\mathbb{Z}_T$ with $T=2^{32}$.
We refer to Appendix~\ref{appendix:PSA-CommCost} for more details about this
communication cost.
The main drawbacks of this approach are its high computation and communication costs and the fact that
aggregated updates are revealed to the aggregator server.

This line of work has been extended by~\cite{SecAgg_BBGL20}.
In~\cite{PracticalSA_GIKM17}, each client speaks to all clients through the central server
to share public keys and encrypted shares.
In~\cite{SecAgg_BBGL20}, they prove that each client is only required
to share public keys and encrypted shares with part of the clients.
That permits to reduce the communication and computation costs
while keeping the privacy guarantees.
For instance, if they have $10^8$ clients, each client is only required to speak to $150$ clients.
Unfortunately, this improvement is not applicable to collaborative setting
where we have less than 10 clients (see Theorem $3.10$ in~\cite{SecAgg_BBGL20}).
In collaborative setting, each client is required to speak with all clients.

To the best of our knowledge, only one proposal combines neural network update 
compression with secure aggregation~\cite{EaSTLy_DCSW20}.
In this work, the authors make use of \texttt{TernGrad}~\cite{TernGrad_WXYW17} 
for model update compression 
and subsequently suggest two methods to perform the secure aggregation:
one based on HE and the second based on TSS.
A vector compressed with \texttt{TernGrad} consists of 
the multipication of a scalar factor by a ternary vector with values in $\{-1, 0, 1\}$.  
In their secure aggregation, the authors propose to 
securely evaluate the sum of the ternary vectors and subsequently each 
client updates their local model from this aggregated ternary vector and according to 
their own local factor.
With this aggregation, contrary to traditional FL methods, each client obtains a 
different intermediate model at the beginning of each federated round and at the end of training.
It is not explicited how these different models should be used or aggregated in one model.
Contrary to this approach, we focus on the federated
training of a \emph{single} model across decentralized clients.
With respect to computation cost, the authors propose two different approaches 
to perform the secure aggregation, each of which requires its own flavor of 
costly cryptographic procedures. 
For the version based on TSS, the implementation requires the evaluation of some polynomial interpolations,
while their version based on Paillier HE requires 
some expensive homomorphic encryption and decryption operations.
For communication costs, the authors do not take into account the sparsity 
of the ternary vector and thus, their communication cost is not optimal with
respect to the significant information contained in the model update. 
More precisely, the total communication cost in bits per round for their secure 
aggregation is equal to
\begin{align}
    & 2 S C N \lceil \log_2 (2C + 1) \rceil \tag{for TSS-based Secure Agg.}\\
    & C N (prec + pad) \tag{for HE-based Secure Agg.}
\end{align}
where $prec = 24$ and $pad = 8$ are HE parameters used to avoid overflows.

In our proposed method, we adapt a compressed FL learning strategy to be more amenable to secure computation
in order to obtain a secure aggregation which has a lower communication and computation costs than reported in~\cite{PracticalSA_GIKM17,EaSTLy_DCSW20}.
Table~\ref{tab:comm-cost-estimations} summarizes the communication costs of the secure aggregations from
\cite{PracticalSA_GIKM17,EaSTLy_DCSW20}.

\section{Experiments}
\label{sec:experiments}

\begin{table*}[t!]
\small
\centering
\begin{tabular}{lcrrrrcr}
    \toprule
    {Strategy} & {Sec. Agg.} & \multicolumn{4}{c}{Parameters} & \multicolumn{2}{c}{\textbf{Metrics}} \\
    \cmidrule(l{2pt}r{2pt}){3-6} \cmidrule(l{2pt}r{2pt}){7-8}
     &  & {E} & {$R$} & {$\rho$} & {$\mathbb{E}~|V|$} & \textbf{Comp. cost} & \textbf{Comm.} \\
    \midrule
    Uncompressed \texttt{FedAvg} & {\xmark}  & $100$ &  $15$ &   {-} &     {-} & + & $35.31$MB\\ 
    \texttt{FedAvg}  w/ \texttt{TopBinary} \& \texttt{DirectAgg} & {\xmark}  & $100$ &  $17$ &   $0.10$ &   {-} & + & $7.23$MB \\ 
    \texttt{FedAvg}  w/ \texttt{TopBinary} \& \texttt{SepAgg} & {\xmark}  & $100$ &  $17$ &   $0.10$ &   $18\,253$ & + & $2.05$MB \\ 
    
    \midrule
    \cite{PracticalSA_GIKM17} & {\cmark} & $100$ &   $15$ & {-}  &  {-} & +++ & $35.37$MB \\
    
    \cite{EaSTLy_DCSW20} with TSS & {\cmark} & $100$ & $17$ &  {-} &  {-} & ++ & $10.00$MB \\
    
    \cite{EaSTLy_DCSW20} with HE & {\cmark} & $100$ &   $17$ &  {-} &  {-}  & +++ & $80.03$MB \\
    
    \midrule
    Our method w/ \texttt{NoUnion} & {\cmark} & $100$ &    $17$ & $0.10$ &  {-} & + & $10.01$MB \\
    
    \textbf{Our method w/ \texttt{SecUnion} (q=1)} & {\cmark} & $\mathbf{100}$ &    $\mathbf{22}$ & $\mathbf{0.10}$ & $\mathbf{14\,344}$  & \textbf{+} & $\mathbf{6.25}$\textbf{MB} \\
    Our method w/ \texttt{SecUnion} (q=5) & {\cmark} & $100$ &    $17$ & $0.10$ &  $18\,037$ & + & $15.43$MB \\
    
    Our method w/ \texttt{PartialSecUnion} & {\cmark} & $100$ & $17$ & $0.10$ & $18\,253$ & + & $10.46$MB \\ 
    
    Our method w/ \texttt{PlaintextUnion} & {\cmark} & $100$ &    $17$ & $0.10$ & $18\,253$ & + & $4.21$MB \\ 
    
   \bottomrule
\end{tabular}\\
\vspace{1ex}
(a) LeNet on MNIST (98\% Acc. Target) \\
\vspace{3ex}
\begin{tabular}{lcrrrrcr}
    \toprule
    {Strategy} & {Sec. Agg.} & \multicolumn{4}{c}{Parameters} & \multicolumn{2}{c}{\textbf{Metrics}} \\
    \cmidrule(l{2pt}r{2pt}){3-6} \cmidrule(l{2pt}r{2pt}){7-8}
     &  & {E} & {$R$} & {$\rho$} & {$\mathbb{E}~|V|$} & \textbf{Comp. cost} & \textbf{Comm.} \\
    \midrule
    Uncompressed \texttt{FedAvg} & {\xmark} & $100$ &  $101$ &   {-} &  {-} & + & $6.61$GB \\
    \texttt{FedAvg} w/ \texttt{TopBinary} \& \texttt{DirectAgg} & {\xmark} & $100$ &  $110$ &   $0.10$ &  {-} & + & $1.60$GB \\
    \texttt{FedAvg} w/ \texttt{TopBinary} \& \texttt{SepAgg} & {\xmark} & $100$ &  $110$ &   $0.10$ &  $663\,630$ & + & $0.41$GB \\
    
    \midrule
    \cite{PracticalSA_GIKM17} & {\cmark} & $100$ & $101$ &  {-} &  {-} & +++ & $6.61$GB \\
    
    \cite{EaSTLy_DCSW20} with TSS & {\cmark} & $100$ & $110$ &  {-} &  {-} & ++ & $1.80$GB \\
    
    \cite{EaSTLy_DCSW20} with HE & {\cmark} & $100$ &   $110$ &  {-} &  {-} & +++ & $14.40$GB \\
    
    \midrule
    Our method w/ \texttt{NoUnion} & {\cmark} & $100$ &  $110$ & $0.10$ & {-} & + & $1.80$GB \\
    
    \textbf{Our method w/ \texttt{SecUnion} (q=1)} & {\cmark} & $\mathbf{100}$ &  $\mathbf{114}$ & $\mathbf{0.10}$ & $\mathbf{574\,599}$ & \textbf{+} & $\mathbf{1.08}$\textbf{GB} \\
    Our method w/ \texttt{SecUnion} (q=5) & {\cmark} & $100$ &  $110$ & $0.10$ & $657\,164$ & + & $2.92$GB \\
    
    Our method w/ \texttt{PartialSecUnion} & {\cmark} & $100$ &  $110$ & $0.10$ & $663\,630$ & + & $2.03$GB \\
    
    Our method w/ \texttt{PlaintextUnion} & {\cmark} & $100$ &  $110$ & $0.10$ & $663\,630$ & + & $0.90$GB \\
    
    \bottomrule
\end{tabular}\\
\vspace{1ex}
(b) AlexNet on CIFAR-10 (80\% Acc. Target) \\
\vspace{2ex}
\caption{%
    Total communication cost 
    to reach the targeted accuracy with \texttt{FedAvg} strategy
    and our approach on five clients.\label{tab:all_xps}}
\end{table*}

In this section, we perform experiments on $C = 5$ clients and $S = 2$ servers,
leading to a moderately sized \emph{collaborative federated system},
representative of \textit{e.g.} healthcare FL applications.

\subsection{Datasets and Model Architectures}
To demonstrate the utility of our proposed approaches, we conduct a series
of image classification experiments on MNIST~\cite{MNISTdatabse} 
and CIFAR-10~\cite{CIFAR10database}. For client data, we use 
i.i.d. dataset partitions to mimic clients holding data samples from 
the same data distribution\footnote{%
    Such homogeneity cannot be expected in practice 
    but data heterogeneity is out of the scope of the current work.
}.
The MNIST dataset is composed of a 60k training image dataset and a 10k test image dataset.
For MNIST experiments, the training dataset is randomly shuffled
and then partitioned into a training dataset of 50k samples and a validation dataset of 10k samples.
Then, the training dataset is partitioned into five sets, and each is given to one
of the $C=5$ clients, thus we have 10k samples per client.
We similarly partition the 40k training dataset of CIFAR10.

For our MNIST experiments, we train a simple LeNet-5 network~\cite{LeNet_LBBH98}
which consists of $N = 61,706$ trainable parameters.
This architecture is detailed in Table~\ref{tab:LeNetArchitecture} in Appendix~\ref{appendix:network_architecture}.
For CIFAR-10, we train the same AlexNet~\cite{AlexNet_KSH17} architecture as demonstrated in the original
\texttt{FedAvg} work~\cite{CommEfficient_MMRH17},
which consists of $N = 1,756,426$ trainable parameters.
This architecture is detailed in Table~\ref{tab:AlexNetArchitecture} in Appendix~\ref{appendix:network_architecture}.
We also utilize the same image preprocessing
pipeline as was proposed in this work, e.g. random flipping, and color normalization.

The goal of our experiments is to demonstrate that with our proposed method, 
compared to \texttt{FedAvg}~\cite{FL_MMRA16} with no compression,
we obtain a computation and communication efficient secure aggregation
without compromising accuracy. We note that more modern neural architectures could be 
utilized, as well as larger datasets across more varied tasks. However, as our approach 
is agnostic to the underlying task, we utilize these simpler ML experiments for the sake of 
clarity of our demonstration.

\subsection{Hyperparameter Selection}
\label{subsection:hyperparameter_selection}
To select appropriate local model training hyperparameters,
we perform a grid search on the entire training dataset without FL (i.e. a single client),
evaluating performance on the held-out validation dataset.
For MNIST, we obtained an accuracy of $99.2\%$ with a learning rate of $0.01$,
a batch size of $64$ and a momentum of $0.9$.
For CIFAR-10, we obtained an accuracy of $80.6\%$ with a learning rate of $0.10$,
a batch size of $64$ and no momentum.
Although this approach cannot be utilized in a practical FL use-case, 
we use it for our FL experiments
and defer optimal hyperparameter tuning to future works.
Indeed, we stress that thanks to the communication efficiency of the proposed approach,
hyperparameter optimisation in FL could be eased by the proposed method.

\subsection{Targeted accuracy}
One of main constraints is to not compromise the prediction performance obtained using our
proposed techniques in comparison to non-secure non-compressed \texttt{FedAvg}.
The accuracy obtained with non-secure non-compressed \texttt{FedAvg}
training on five clients with the selected hyperparameters from Sec~\ref{subsection:hyperparameter_selection}
is equal to $99.0\%$ for MNIST (with $E=100$ and $50$ epochs)
and $81.3\%$ for CIFAR-10 (with $E=100$ and $150$ epochs).
From those performances, we selected a desired performance level to achieve:
$98\%$ accuracy for MNIST and $80\%$ accuracy for CIFAR-10.

\subsection{Results}

We conduct several ML training experiments to define the number of federated rounds $R$
required to obtain the targeted accuracy ($98\%$ for MNIST and $80\%$ for CIFAR-10)
for \texttt{FedAvg} and our proposed methods.
For experiments with \texttt{TopBinary} coding, we set $\rho = 0.1$ and we also report
the mean size of the aggregated union of indices over the course of federated training,
$\mathbb{E}~|V|$.
For the method presented in~\cite{PracticalSA_GIKM17}, no update compression is used
and thus, the number of federated round $R$ to reach the targeted accuracy is equal to that of uncompressed \texttt{FedAvg}.
The method presented in~\cite{EaSTLy_DCSW20} uses the \emph{TernGrad} model-update compression technique,
which is similar to our \texttt{TopBinary} coding.
Thus, we assume that this technique with \emph{TernGrad} compression requires a similar number of
federated rounds to obtain the same level of accuracy as our proposed techniques.
We recall that these approaches introduced in~\cite{EaSTLy_DCSW20} cannot be used in practice
because at the end of the FL training, each client obtains a different trained models
and we do not know how to aggregate these models into one model without compromising accuracy.

To compare the computation cost, we compare the cost of the most expensive operations.
For non secure aggregations and our proposed methods,
only additions (noted "+" in Table~\ref{tab:all_xps}) are performed.
\cite{EaSTLy_DCSW20} with TSS is based on polynomial interpolation (noted "++" in Table~\ref{tab:all_xps})
which is more expensive than additions.
\cite{EaSTLy_DCSW20} with HE and \cite{PracticalSA_GIKM17} are based on the most expensive operations
respectively Paillier HE and asymmetric encryption (noted "+++" in Table~\ref{tab:all_xps}).

Finally, we analytically estimate
the communication costs associated with these strategies 
(see Table~\ref{tab:comm-cost-estimations}),
the results of which are detailed in Table~\ref{tab:all_xps}.
These experimental results showcase the low total communication cost of our 
proposed compressed secure aggregation approach compared to other 
techniques with or without secure aggregation.
This communication cost can be further reduced with more fine tuning, as 
higher compression rates (i.e. $\rho < 0.02$) may be achievable for the 
same performance level \cite{SparseBinaryCompression_SWMS18}.

We emphasize that our \texttt{SecUnion} protocol is not correct, since some indices are missing in the union result.
In this protocol, the parameter $q$ is used to manage this amount of false negatives:
higher is $q$, smaller is the amount of false negatives.
With $q=1$, we observed that many selected gradients are not used during the secure aggregation due to the incorrectness of the \texttt{SecUnion} protocol.
Among all our experiments, our method with \texttt{SecUnion} ($q=1$) achieves
the best trade-off between privacy and communication cost.
In addition, this strategy has a lower communication cost than~\cite{PracticalSA_GIKM17,EaSTLy_DCSW20}.

\section{Conclusion}

In this article, we design a set of secure and efficient protocols for federated learning.
These protocols are based on quantization to highly reduce the communication cost of the aggregation
without compromising accuracy.
We adapt quantizated federated learning techniques to secure computation in order to be able
to efficiently and securely evaluate the aggregation steps during a federated training.
Thanks to these adaptations, our secure aggregation protocols solely rely on 
additions, without any expensive cryptographic operations, and thus have a 
very low computation cost.
Since our secure aggregations are only based on additions, it can 
be easily implemented by non-cryptographic experts.
We also prove that our protocols are private against malicious 
adversaries and correct against semi-honest adversaries.
Finally, our experiments show their efficiency in real use cases.
Compared to prior works, they have a lower computation cost,
as no expensive operations are used,
and a lower communication cost thanks to model update compression and lightweight SMC protocols.

Future works could investigate the interplay between the proposed protocols and heterogeneity,
which is another important challenge for FL.
Other compression approaches could also be used to further reduce the communication cost.
Finally, there still exist large
communication differences between the different security profiles of the secure
union techniques we propose. It would be of interest to find
better trade-offs in the security perimeter of such protocols against communication
requirements.

\section*{Acknowledgements}
We thank Jean du Terrail whose comments helped improve and clarify this manuscript.

\bibliography{references}
\bibliographystyle{mlsys2022}

\newpage
\appendix
\section{Practical Secure Aggregation \cite{PracticalSA_GIKM17} Communication Cost}
\label{appendix:PSA-CommCost}

The communication cost of the Practical Secure Aggregation protocol is presented in Section 7.1 of \cite{PracticalSA_GIKM17}.
For each federated round, each client transmits to the Server
$$2n \cdot a_K + (5n-4) \cdot a_S+m \lceil \log_2 R  \rceil \text{ bits},$$
where $n$ is the number of clients, $m$ is the number of trainable parameters, $a_K= 256$ 
is the number of bits in a public key, $a_S=256$ is the number of bits in an encrypted share,
and $\mathbb{Z}_R$ is the field in which all operations are evaluated.
And at the end of the federated round, the server transmits the aggregated model parameters to each client
at a cost of $m \lceil \log_2 R  \rceil \text{ bits}$.
Thus, the total communication cost for one federated round is equal to 
$$n \cdot (2n \cdot a_K + (5n-4) \cdot a_S+ 2m \lceil \log_2 R  \rceil) \text{ bits}.$$

Using our notation, the total communication cost for a single federated round is thus given as
$$C \cdot (2C \cdot a_K + (5C-4) \cdot a_S+ 2N \lceil \log_2 R  \rceil) \text{ bits},$$
where $C$ is the number of clients, $N$ is the number of trainable parameters and $\mathbb{Z}_R$ is the field in which all operations are evaluated.

\section{Further details on \texttt{TopBinary} Aggregation}
\label{appendix:separate-aggregation}
Let us recall that we want to operate
on the aggregation of some real-valued vectors $X_i \in \mathbb{R}^N$,
of which we have~$C$ different realizations.
From each of these vectors, we calculate one new vector
$D_i = {\rm sgn}\circ\text{Top-}k(X_i)$
and one scalar value
$\alpha_i = \frac{||X_i||_2}{||D_i||_2}$,
such that $||\alpha_i D_i||_2 = ||X_i||_2$.  
Because of the ${\rm sgn}\circ\text{Top-}k$ operation, the result \emph{sign} vector belongs
to the space $D_i \in \{-1, 0, 1\}^{N}$.
Now, since the value of $k$ is fixed,
we know, trivially, the number of non-zero entries: $k$.  And, since
the non-zero entries are $\pm 1$, we know 
$||D_i||_2 = \sqrt{k} \quad \forall i$,
thus $\alpha_i = \frac{1}{\sqrt{k}} ||X_i||_2$.

Let us now define the different variants of the aggregation again,
\begin{align}
    \bar{D} &= \frac{1}{C}\sum_{i=1}^C \alpha_i D_i
    = \frac{1}{C\sqrt{k}} \sum_{i=1}^C L_i D_i, \tag{\texttt{DirectAgg}}\\
    \tilde{D} &= \frac{1}{C^2}\sum_{i=1}^C \alpha_i \sum_{i=1}^C D_i 
    = \frac{1}{C \sqrt{k}}\cdot\frac{1}{C}\sum_{i=1}^C L_i \cdot \sum_{i=1}^C D_i, \tag{\texttt{SepAgg}}
\end{align}
where $L_i \triangleq ||X_i||_2$ is the Euclidean norm of the original
vector $X_i$. Our desire is to compute the value of $\bar{D}$,
however, for the efficiency of SMC, we may only make use of 
the estimate $\widetilde{D}$.
Given this, how far off will we be, 
and what are the effects of using such an estimate?

Let $\widetilde{L} \triangleq \frac{1}{C}\sum_{i=1}^C L_i$, i.e. the average norm
of all of the original vectors $X_i$, then
\begin{align}
    \widetilde{D} 
        &= \frac{\widetilde{L}}{\sqrt{k}} \cdot \frac{1}{C} \sum_{i=1}^C D_i 
        = \frac{1}{C\sqrt{k}} \sum_{i=1}^C \widetilde{L} D_i,
\end{align}
where we can see that the average norm $\widetilde{L}$ serves as
an approximation for the individual norms $L_i$.
We now evaluate the MSE difference between the \texttt{SepAgg} and \texttt{DirectAgg} as
\begin{align}
||\bar{D} - \widetilde{D}||_2^2 &=
    \frac{1}{C^2 k} \sum_{j=1}^N \left( \sum_{i=1}^C(L_i - \widetilde{L}) D_i[j]\right)^2,
\end{align}
which shows us that if the expectation of the difference~$L_i - \widetilde{L}$ is small,
we can expect that the \texttt{SepAgg} will be an adequate estimation of the
direct aggregation. 

\section{Further details on \texttt{SecUnion} protocol}
\label{appendix:secure_union}

For the analysis of the \texttt{SecUnion} protocol, let us assume that 
each client $i$ has selected independently and uniformly at random
their list of $k$ indices $V_i \subset [1,N]$,
then, for each client the probability of selecting the index $x \in [1, N]$ is 
equal to $k / N$.
Thus, the probability that exactly $t$ clients have selected the index $x$ is
\begin{equation}
    P(N, k, C, t) = \binom{C}{t} \left( \frac{k}{N} \right)^t \left( \frac{N-k}{N} \right)^{C-t}.
\end{equation}

\subsection{Expected number of false negatives}

When using the \texttt{SecUnion} protocol,
false negatives occur only when at least 2 clients have selected the same index $x$
and $\sum_{i=1}^C A_i[x] \mod 2^q$ is equal to zero. 
The probability that at least 2 clients have selected the same index $x$ is equal to
$\text{Pr}[t \geq 2 | N, k, C] = 1 - P(N, k, C, 0) - P(N, k, C, 1)$.
When one index is selected by at least 2 clients, the probability to obtain a false negative 
for this index is equal to~$1 / 2^q$.
Thus, the expected number of false negatives in $V$ is given by
$N \times P[t \geq 2 | N, k, C] \times \frac{1}{2^q}$.
The Table~\ref{tab:false_negatives} presents the average amount of false negatives in the \texttt{SecUnion} protocol for different parameters.

\begin{table}[!h]
    \small
    \centering
    \begin{tabular}{rrrrrr}
        \toprule
        {$N$} & {$\rho$} & {$k=\lfloor N \cdot \rho \rfloor$} & {$C$} & {$q$} & {$\mathbb{E}$}\\
        \midrule
        $61\,706$ & $0.1$ & $6\,170$ & $5$ & $1$ & $2\,513$ \\
        $61\,706$ & $0.1$ & $6\,170$ & $5$ & $5$ & $157$ \\
        $61\,706$ & $0.1$ & $6\,170$ & $5$ & $10$ & $5$ \\
        \midrule
        $1\,756\,426$ & $0.1$ & $175\,642$ & $5$ & $1$ & $71\,539$ \\
        $1\,756\,426$ & $0.1$ & $175\,642$ & $5$ & $5$ & $4\,471$ \\
        $1\,756\,426$ & $0.1$ & $175\,642$ & $5$ & $10$ & $140$ \\
        \bottomrule
    \end{tabular}
    \caption{Average amount of false negatives $\mathbb{E}$ in the \texttt{SecUnion} protocol.
    \label{tab:false_negatives}}
\end{table}

\subsection{One client learns that he alone has selected a given index}

Let us assume the client $i$ observes that $A[x] = R_{i,x}$ for one index~$x$.
The client $i$ is not alone to have selected this index $x$, when at least two other 
clients have selected this index $x$ and the sum of the random values 
$R_{i,x}$ selected by these other clients is equal to $0$ in~$\mathbb{Z}_{2^q}$.
The probability that at least two other clients have selected this index $x$ 
is equal to $\text{Pr}[t\geq 2 | N, k, C-1]$.
When several clients have selected this index $x$, 
the sum of their random values~$R_{i,x}$ will be equal to $0$ with probability~$1/2^q$.
Finally, if one client~$i$ observes that~$A[x] = R_{i,x}$, then the client can 
infer he alone has selected this index~$x$, with probability~$1 - \frac{1}{2^q}~\text{Pr}[t \geq 2 | N, k, C-1]$.
This probability is evaluated for some parameters in Table~\ref{tab:alone_select_index}.

\begin{table}[!h]
    \small
    \centering
    \begin{tabular}{rrrrrr}
        \toprule
        {$N$} & {$\rho$} & {$k=\lfloor N \cdot \rho \rfloor$} & {$C$} & {$q$} &  {Probability}\\
        \midrule
        $61\,706$ & $0.1$ & $6\,170$ & $5$ & $1$ & $97.39\%$ \\
        $61\,706$ & $0.1$ & $6\,170$ & $5$ & $5$ & $99.84\%$ \\
        $61\,706$ & $0.1$ & $6\,170$ & $5$ & $10$ & $99.99\%$ \\
        \midrule
        $1\,756\,426$ & $0.1$ & $175\,642$ & $5$ & $1$ & $97.39\%$ \\
        $1\,756\,426$ & $0.1$ & $175\,642$ & $5$ & $5$ & $99.84\%$ \\
        $1\,756\,426$ & $0.1$ & $175\,642$ & $5$ & $10$ & $99.99\%$ \\
        \bottomrule
    \end{tabular}
    \caption{If one client $i$ observes that $A[x] = R_{i,x}$, we evaluate the probability that this client alone has selected this index~$x$.
    \label{tab:alone_select_index}}
\end{table}

\section{Neural Networks Architectures}
\label{appendix:network_architecture}

In this appendix, we present the LeNet architecture used for experiments on MNIST (see Table~\ref{tab:LeNetArchitecture})
and the AlexNet architecture used for experiments on CIFAR-10
(see Table~\ref{tab:AlexNetArchitecture}).

\begin{table}[!h]
\small
\centering
\setlength\tabcolsep{2pt}
\begin{tabular}{lccccc}
\toprule
Layer Type & Filters & Kernel & Stride & Padding & Output Shape \\
\midrule
Conv2D+ReLU & 6 & (5,5) & (1,1) & (2,2) & $28 \times 28 \times 6$ \\
Max Pooling & & (2,2) & & & $14 \times 14 \times 6$ \\
Conv2D+ReLU & 16 & (5,5) & (1,1) &  & $10 \times 10 \times 16$ \\
Max Pooling & & (2,2) & & & $5 \times 5 \times 16$ \\
FC+ReLU & & & & & $120$ \\
FC+ReLU & & & & & $84$ \\
FC & & & & & $10$ \\
\bottomrule
\end{tabular}
\caption{Architecture of the LeNet used for experiments on MNIST.
\label{tab:LeNetArchitecture}}
\end{table}

\begin{table}[!h]
\small
\centering
\setlength\tabcolsep{2pt}
\begin{tabular}{lccccc}
\toprule
Layer Type & Filters & Kernel & Stride & Padding & Output Shape \\
\midrule
Conv2D+ReLU & 64 & (5,5) & (1,1) & (2,2) & $32 \times 32 \times 64$ \\
Max Pooling & & (3,3) & (2,2) & (1,1) & $16 \times 16 \times 64$ \\
\multicolumn{2}{l}{LocalResponseNorm} & & & & $16 \times 16 \times 64$ \\
Conv2D+ReLU & 64 & (5,5) & (1,1) & (2,2) & $16 \times 16 \times 64$ \\
\multicolumn{2}{l}{LocalResponseNorm} & & & & $16 \times 16 \times 64$ \\
Max Pooling & & (3,3) & (2,2) & (1,1) & $8 \times 8 \times 64$ \\
FC+ReLU & & & & & $384$ \\
FC+ReLU & & & & & $192$ \\
FC & & & & & $10$ \\
\bottomrule
\end{tabular}
\caption{Architecture of the AlexNet used for experiments on CIFAR-10. The LocalResponseNorm layers have the following parameters: $size=4$, $\alpha=0.001/9$, $\beta=0.75$, $k=1$.
\label{tab:AlexNetArchitecture}}
\end{table}

\section{Privacy Proof of SecureSum Algorithm}
\label{appendix:privacy-proof}

\begin{theorem}
    The SecureSum Protocol is private against a
    malicious adversary controlling at most $S-1$ servers and any number of clients.
    \label{th:privacy}
\end{theorem}

The privacy proof of this theorem is mainly based on the following lemma.

\begin{lemma}
Let $b \in \mathbb{Z}_m$. If $r$ is a uniform random element in $\mathbb{Z}_m$,
then $b+r$ is also a uniform random element in $\mathbb{Z}_m$
(even if $b$ is not a uniform random element in $\mathbb{Z}_m$).
\label{lemma:uniformly-random}
\end{lemma}

We will now prove Theorem~\ref{th:privacy}.

\begin{proof}
Let $A$ be an adversary controlling $p<S$ servers,
$Server_{j_1}$, \dots, $Server_{j_p}$, and
$q \leq C$ clients, $Client_{i_1}$, \dots, $Client_{i_q}$.
In this proof, we will use the following notations for
different sets of servers and clients,
$\begin{cases}
    J &= \{j_1, ..., j_p\} \\
    I  &= \{i_1, ..., i_q\}
\end{cases} \text{ and }
\begin{cases}
    \overline{J} &= \{j~|~ j \in [1, S] \text{ and } j \notin J\} \\
    \overline{I} &= \{i~|~i \in [1,C] \text{ and } i \notin I\}
\end{cases}$.
We will prove the privacy of our protocol in the ideal-real paradigm~\cite{SimulationGame_L17}.
We will first describe the simulator $S_A$, who simulates the 
view of the adversary $A$ in the ideal world.
Then, we will prove that the views of the adversary $A$
in the ideal world and in the real world are indistinguishable.
We will split the proof in three cases, depending on the number of 
clients the adversary controls: 
I. no clients ($q = 0$),
II. at least one and up to $C-2$ clients ($1 \leq q \leq C-2$), and
III. all clients, or all clients except one ($C-1 \leq q \leq C$).

\vspace{1em}

\centerline{\textsc{Case I: $q = 0$}}

\noindent \textbf{Simulation.}
$S_A$ picks $p \times C$ uniformly random vectors $\tilde{R}_{i,j}, \forall  (i,j) \in [1,C] \times J$ into  $\mathbb{Z}_m^n$ and sends them to the adversary $A$.

\noindent \textbf{Indistinguishability of Views.}
We will now prove that the views of $A$ in the real world and in the ideal world are indistinguishable.
In the real world, the adversary $A$ receives from honest parties the following messages
\begin{equation}
    \forall (i,j) \in [1,C] \times J, R_{i,j}.
\end{equation}
If $S \notin J$, then by construction,
$\forall (i,j) \in [1, C] \times J, R_{i,j}$ 
are independently and uniformly random in $\mathbb{Z}_m^n$ 
(see L.2 in Alg.~\ref{algo:secure-sum}).
Thus, the views of the adversary~$A$ in the real world and in the ideal world are indistinguishable.

If $S \in J$, since $J \subsetneq [1, S]$, then $[1, S-1] \bigcap \overline{J}$ is not empty.
Let $j^*$ be one element of $[1, S-1] \bigcap \overline{J}$, and 
with it rewrite~$R_{i, S}$ as
\begin{equation}
    R_{i,S} = \left( X_i - \sum_{j \in [1, S-1] \backslash j^*} R_{i,j} \right)  - R_{i,j^*} \mod m.
\end{equation}
By construction, $ \forall i \in [1,C]$, the values $R_{i, j^*}$ 
are uniformly random in $\mathbb{Z}_m^n$
(see L.2 in Alg.~\ref{algo:secure-sum}).
By applying Lemma~\ref{lemma:uniformly-random},~$\forall i \in [1,C]$,
the values $R_{i, S}$ are uniformly random in $\mathbb{Z}_m^n$.

This proves that all messages received by $A$ are 
independently and uniformly random in $\mathbb{Z}_m^n$, and 
thus, the views of the adversary~$A$ in the real world and 
in the ideal world are indistinguishable.

\vspace{1em}

\centerline{ \textsc{Case II: $1 \leq q \leq C-2$} }

\noindent \textbf{Simulation.}
$S_A$ picks $(C-q) \times p + (S-p)$ uniformly random vectors 
into  $\mathbb{Z}_m^n$ noted 
$\tilde{R}_{i,j}, \forall (i,j) \in \overline{I} \times J$  
and $\tilde{R}_j, \forall j \in \overline{J}$
and sends them to the adversary $A$.

\noindent \textbf{Indistinguishability of Views.}
In the real world, the adversary $A$ receives from honest parties the following messages:
\begin{equation}
    \begin{cases}
    R_{i,j}, \quad \forall (i,j) \in \overline{I} \times J \\
    R_j, \quad \forall j \in \overline{J}
    \end{cases}
\end{equation}
With a similar proof to \textsc{Case I}, we can demonstrate that
$\forall (i,j) \in \overline{I} \times J$, the values $R_{i,j}$ 
are independently and uniformly random in $\mathbb{Z}_m^n$.
Now, let $i^*$ be an element in $\overline{I}$, where $\overline{I}$ 
is not empty because the adversary controlled at most $(C-2)$ clients.
Then, $\forall j \in \overline{J}$, we can rewrite $R_j$ as
$R_j = \left(\sum_{i \in [1,C] \backslash i^*} R_{i,j}\right) + R_{i^*,j} \mod m$.

By construction, for all $j \in \overline{J}$, $R_{i^*, j}$ are independently and uniformly random in $\mathbb{Z}_m^n$
(see L.2 in Alg.~\ref{algo:secure-sum}).
By applying Lemma~\ref{lemma:uniformly-random},
$\forall j \in \overline{J}$, the values 
$R_j$ are uniformly random in~$\mathbb{Z}_m^n$.

This proves that all messages received by $A$ are independently and 
uniformly random in~$\mathbb{Z}_m^n$, and
thus, the views of the adversary $A$ in the real world and in 
the ideal world are indistinguishable.

\vspace{1em}

\centerline{\textsc{Case III: $C-1 \leq q \leq C$}}

Assume that the adversary controls all the clients ($q=C$).
Then, the adversary knows all inputs and outputs.
Thus, our protocol is private according to the privacy definition 
outlined in Sec.~\ref{subsec:SMC}.

Now assume that the adversary controls all clients \emph{except} 
the client $i^*$.
Then, the inputs of the parties controlled by the adversary are 
$\{X_i\}_{i \in [1,C] \backslash i^*}$ and the outputs of the parties
controlled by the adversary are $X = \sum_{i \in [1, C]} X_i$.
The adversary can easily infer the input of the client $i^*$ as
$X_{i^*} = X - \sum_{i \in [1,C] \backslash i^*} X_i \mod m$.
Thus, the adversary cannot learn anything more, because they already know
everything. 
Thus, our protocol is private according to the privacy definition in 
Sec.~\ref{subsec:SMC}.

\end{proof}

We just proved that the SecureSum protocol (Alg.~\ref{algo:secure-sum}) is private
against a malicious adversary controlling at most~$(S-1)$ servers and any number of clients.
That means that a malicious adversary cannot learn more than
what he can infer from the inputs and outputs of the parties they control.
In practice, if the adversary controls \emph{all} clients, the adversary already knows everything.
If the adversary controls all clients \emph{except one}, he can deduce the inputs of this honest client
from the inputs and outputs of the clients that are under his control.
If the adversary controls at most~$(C-2)$ clients, the adversary can infer the sum of the inputs
of the honest clients from the inputs and outputs of the clients that he controls,
but not the individual inputs of the honest parties. In this manner, the security of the
individual clients is preserved, however the sum of vectors of honest parties is not.

\end{document}